\documentclass[onecolumn]{IEEEtran}

\usepackage{algorithmic}
\usepackage{algorithm}
\usepackage{hyperref}

\usepackage{latexsym}
\usepackage{graphicx}
\usepackage{bm}
\usepackage{amsfonts,amsmath,amsthm}

\newtheorem{thm}{Theorem}[section]
\newtheorem{lem}[thm]{Lemma}
\newtheorem{cor}[thm]{Corollary}
\newtheorem{prop}[thm]{Proposition}

\newtheorem{defn}{Definition}[section]
 \usepackage{cite}

\begin{document}
 
\title{Combining Drift Analysis and Generalized Schema Theory to Design Efficient Hybrid and/or Mixed Strategy EAs}
\author{Boris Mitavskiy and Jun He \thanks{Boris Mitavskiy and Jun He  are with  Department of Computer Science, Aberystwyth University, Aberystwyth, SY23 3DB, UK. }}

\maketitle

\begin{abstract}
%\boldmath
Hybrid and mixed strategy EAs have become rather popular for tackling various complex and NP-hard optimization problems. While empirical evidence suggests that such algorithms are successful in practice, rather little theoretical support for their success is available, not mentioning a solid mathematical foundation that would provide guidance towards an efficient design of this type of EAs. In the current paper we develop a rigorous mathematical framework that suggests such designs based on generalized schema theory, fitness levels and drift analysis. An example-application for tackling one of the classical NP-hard problems, the ``single-machine scheduling problem" is presented.
\end{abstract}


\section{Introduction}\label{IntroSect}
In recent years hybrid and/or mixed strategy EAs are frequently applied to tackle various NP-hard optimization problems. Such algorithms exploit a variety of different recombination, mutation, and selection operators where these operators are chosen with distinct probabilities depending on the current population, the individuals that are selected for recombination or mutation, and, sometimes, the time when the population arises. According to the empirical evidence, many such algorithms are rather successful in practice (see, for instance, \cite{galinier1999hybrid} and \cite{preux1999towards} and \cite{dong2007evolutionary}). At the same time, rather little theoretical support or, the more so, general guidance for the design of such EAs, exists in the literature. This work is largely motivated by a special case design of a hybrid $1+1$ EA on a single machine scheduling problem in \cite{MitavJunSchedule}, however, in this paper, we aim to investigate possible methodology for the design of population-based EAs of this type. In the current article we unify the theory of hybrid and mixed strategy EAs into a common mathematical framework. This opens the door to various existent and well-developed mathematical tools such as generalized schema theory, drift analysis and tail inequalities to design hybrid and mixed strategy EAs for various specific problems with polynomial runtime guarantees to encounter a satisfactory solution (such as a solution up to a desirable or allowable approximation ratio). The paper is organized in a straightforward fashion: In section~\ref{MathDescribeSect} we set up a rigorous mathematical framework that incorporates a wide class of hybrid and mixed strategy EAs. Next, in section~\ref{schemaSection}, generalized schema theory is presented. In section~\ref{RecInvarAndDriftSect} the central idea of the article, namely the design of hybrid and mixed strategy EAs for specific optimization problems, based on the notion of auxiliary fitness levels and schemata is provided. These ideas can then be linked with classical tools from applied probability to analyze runtime complexity: for instance, in section~\ref{mainResultsSect}, drift analysis methodology has been applied to analyze the conditions under which the expected runtime bounds are polynomial. Finally, this approach is illustrated with a specific example application: designing a family of hybrid and mixed strategy population-based EAs for the ``single machine scheduling problem" (see \cite{LeslieHall} and \cite{MitavJunSchedule} for a detailed description) with expected polynomial time approximation ratio guarantees.
\section{Mathematical Description of Hybrid and Mixed Strategy EAs}\label{MathDescribeSect}
While the families of recombination, mutation and selection operators are typically independent of population and the iteration time at which the population is encountered, hybrid and mixed strategy EAs exploit several families of recombination, mutation and selection operators. Furthermore, each pair $(P, \, t)$ where $P$ is the population encountered at $t^{\text{th}}$ iteration of a given hybrid or mixed strategy EA, is equipped with a probability distribution over the various families of recombination, mutation and selection transformations. Therefore, mathematically, mixed and hybrid EAs are fully determined in terms of the following parameters:

\textbf{1.} A finite set $\Omega$ of candidate solutions that we call the search space.

\textbf{2.} A tuple of indexed families $$(\{\mathcal{F}_i\}_{i \in I}, \, \{\mathcal{M}_j\}_{j \in J}, \, Sel_q\}_{q \in Q}, \{f_l\}_{l \in L \cup \{0\}})$$ where $I$, $J$, $Q$ and $L$ are various indexing sets (usually finite subsets of $\mathbb{N}$) while $\mathcal{F}_i$, $\mathcal{M}_j$, $Sel_q$ are various families of recombination, mutation and selection operators respectively and $\{f_l\}_{l \in L}$ is the family of fitness functions. Among the fitness functions, only the function $f_0$ is the objective function to be optimized. We will say that the remaining fitness functions are \emph{auxiliary fitness functions}. In practice these may be given implicitly, motivated by certain deterministic algorithms (such as the Jackson rule in case of the single machine scheduling problem). Recombination operators are usually maps $F: \, \Omega^2 \rightarrow \Omega$ that take a pair of individuals $(x, \, y) \in \Omega^2$ to a single offspring $F(x, \, y) \in \Omega$,\footnote{One may also allow general $m$-ary operators. In case when $m = 1$, i.e. when the recombination operators are unary, they are usually known as mutation operators.} while mutation operators are functions $M : \Omega \rightarrow \Omega$. A selection operator is a function $Sel: \left(\Omega^m \right)^2 \rightarrow \Omega^m$ on the set of pairs of populations of size $m$ such that $\forall$ two populations $\vec{x} = (x_1, \, x_2, \ldots \, x_m)$ and $\vec{y} = (y_1, \, y_2, \ldots \, y_m) \in \Omega^m$ all the individuals of the population $Sel(\vec{x}, \, \vec{y})$ are also the individuals of the population $\vec{x}$ or of the population $\vec{y}$: $Sel(\vec{x}, \vec{y}) = (z_1, \, z_2, \ldots \, z_m) \in \Omega^m$ and $\forall \, i$ with $1 \leq i \leq m$ $\exists \, j$ with $1 \leq j \leq m$ such that $z_i = x_j$ or $z_i = y_j$. It is also reasonable to assume that given any subset $S \in \Omega$ such that all the individuals in $S$ appear either in the population $\vec{x}$ or $\vec{y}$, then the number of individuals from the subset $S$ that appear in the population $Sel(\vec{x}, \vec{y})$ does not depend on the specific location (indexing) of the elements from $S$ in the populations $\vec{x}$ and $\vec{y}$. Most certainly it may depend on all other parameters such as the fitness of the various individuals in $S$ and even on whether or not these individuals occur in the population $\vec{x}$ or in the population $\vec{y}$.

\textbf{3.} To every pair $(\vec{x}, \, t)$ where $\vec{x} = (x_1, \, x_2, \ldots \, x_m) \in \Omega^m$ is a population at $t^{\text{th}}$ iteration of the algorithm and to every pair of individuals $(x_i, \, x_j)$ in $\vec{x}$ assign probability distributions $Pr^{\text{RecFamily}}_{(\vec{x}, \, t), \, (x_i, \, x_j)}$ on the set indexing the families of recombination transformations. and for every index $w \in I$, a probability distribution $Pr^{\text{Rec}, \, w}_{(\vec{x}, \, t), \, (x_i, \, x_j)}$ on the set of recombination transformations $\mathcal{F}_w$. For simplicity we shall assume that the pairs $(x_i, \, x_j)$ are sampled uniformly at random (either with replacement or without replacement) from the population $\vec{x}$. Once a pair $(x_i, \, x_j)$ has been selected for recombination, it first selects a family of recombination operators to use according to the probability distribution $Pr^{\text{RecFamily}}_{(\vec{x}, \, t), \, (x_i, \, x_j)}$ and then, once the index $w$ has been chosen, it selects a specified transformation to use according to the probability distribution $Pr^{\text{Rec}, \, w}_{(\vec{x}, \, t), \, (x_i, \, x_j)}$ on the family of recombination transformations $\mathcal{F}_w$. Mutation operators are selected analogously except that this time only a single individual, say $x_i$, is selected uniformly at random from the population $\vec{x}$ and selects a family of mutation transformations $\mathcal{M}_w$ according to a probability distribution $Pr^{\text{MutFamily}}_{(\vec{x}, \, t), \, x_i}$ on the indexing set $J$ of the families of mutation transformations. Afterwards, it selects mutation transformations from the family $\mathcal{M}_w$ according to the probability distribution $Pr^{\text{Mut}, \, w}_{(\vec{x}, \, t), \, x_i}$ on the family of mutation transformations $\mathcal{M}_w$. Likewise, to every pair $(\vec{x}, \, t)$ and a population $\vec{y} \in \Omega^m$ we associate a probability distribution $Pr^{\text{SelFamily}}_{(\vec{x}, \, t), \vec{y}}$ on the indexing set $Q$ of the families of selection transformations, and to every family of selection transformations $Sel_w$ we assign a probability distribution $Pr^{\text{Sel}, \, w}_{(\vec{x}, \, t), \, \vec{y}}$ on the family $Sel_w$ of selection transformations. Once an ``intermediate" population $\vec{y}$ has been obtained from the population $\vec{x}$ upon completion of recombination stage followed by mutation stage, an appropriate family of selection transformations $S_w$ is selected through sampling its index via the probability distribution $Pr^{\text{SelFamily}}_{(\vec{x}, \, t), \vec{y}}$. Afterwards, an appropriate selection transformation $Sel$ is chosen from the family of selection transformations $Sel_w$ via the probability distribution $Pr^{\text{Sel}, \, w}_{(\vec{x}, \, t), \, \vec{y}}$.

A hybrid/mixed strategy EA cycles through the recombination, mutation and selection stages sufficiently long to encounter a satisfactory solution. A single cycle consisting of these three consecutive stages is typically called a single iteration of the algorithm that produces the next generation from the previous one.

\section{Generalized Schema Theory}\label{schemaSection}
In the current section we will establish a very general version of the schema theorem that applies to the types of EAs fitting the framework in the previous section. Suppose we are given any subset $S \subseteq \Omega$ of the search space of an EA and a population $\vec{x}$ at $t^{\text{th}}$ iteration of the algorithm. Recall from the previous section that various individuals from the population $\vec{x}$ are paired up for recombination independently $m$ times with the aim of producing exactly $m$ offsprings. Thus, the probability that an individual from the set $S$ appears at the $i^{\text{th}}$ position of the ``intermediate" population $\vec{x}^{\text{rec}}$ obtained from the population $\vec{x}$ upon completion of recombination can be computed as follows:
\begin{lem}\label{prodProbLem}
Continuing with the notation in the preceding paragraph, for a given pair of individuals $(x_i, \, x_j) \in \vec{x}^2$, let
$$Pr(S \, | \, (x_i, \, x_j)) = \sum_{w \in I}\sum_{F \in \mathcal{F}_w}^{F(x_i, \, x_j) \in S}Pr^{\text{RecFamily}}_{(\vec{x}, \, t), \, (x_i, \, x_j)}(i = w) \cdot Pr^{\text{Rec}, \, w}_{(\vec{x}, \, t), \, (x_i, \, x_j)}(F).$$
Furthermore, let $$Pr^{nonRepl}(S \, | \, \vec{x}, \, t) = \frac{\sum_{i \neq j}Pr(S \, | \, (x_i, \, x_j))}{m(m-1)}$$ and $$Pr^{Repl}(S \, | \, \vec{x}, \, t) = \frac{\sum_{(i, \, j) \in \{1, \, 2, \ldots, m\}}Pr(S \, | \, (x_i, \, x_j))}{m^2}$$
Then, in case when pairs of individuals are selected for recombination independently without replacement, $\forall \, i$ with $1 \leq i \leq m$, the probability that the $i^\text{th}$ individual in the intermediate population $\vec{x}^{\text{rec}}$ is in the set $S$ is $Pr^{nonRepl}(S \, | \, \vec{x}, \, t)$. Likewise, in case when pairs of individuals are selected for recombination independently with replacement, this probability is $Pr^{Repl}(S \, | \, \vec{x}, \, t)$.
\end{lem}
\begin{proof}
According to the general framework in the previous section, if the individuals $x_i$ and $x_j$ have been selected for recombination, $Pr(S \, | \, (x_i, \, x_j))$ is the probability that the offspring individual is an element of the subset $S \subseteq \Omega$ is precisely $Pr(S \, | \, (x_i, \, x_j))$. Since every pair is selected for recombination uniformly at random, the probability that the pair $(x_i, \, x_j)$ has been chosen for recombination and their offspring is in $S$ is $\frac{Pr(S \, | \, (x_i, \, x_j))}{m(m-1)}$ in case of sampling pairs without replacement and $\frac{Pr(S \, | \, (x_i, \, x_j))}{m^2}$ in case of sampling with replacement. The desired conclusion follows now by summing the probabilities of pairwise disjoint events.
\end{proof}
Since recombination takes place independently, the number of individuals in the intermediate population $\vec{x}^{\text{rec}}$ is distributed binomially with success probabilities \\ $Pr^{nonRepl}(S \, | \, \vec{x}, \, t)$ in case of sampling recombination pairs without replacement and $Pr^{Repl}(S \, | \, \vec{x}, \, t)$ in case of sampling recombination pairs with replacement respectively, Chernoff tail inequality (see, for instance, chapter 1 of \cite{DoerrB}) applies and readily tells us the following.
\begin{lem}\label{lemmaChernoff}
Continuing with the notation in lemma~\ref{prodProbLem}, let $N(S, \vec{y})$ denote the random variable counting the total number of individuals in the population $\vec{y}$ of size $m$ that are in the set $S$. Then $\forall \, \delta \in [0, 1]$ $$Pr(N(S, \vec{x}^{\text{rec}}) < (1-\delta) m P) \leq \exp\left(-\delta^2 \frac{m P}{2}\right)$$ while $$Pr(N(S, \vec{x}^{\text{rec}}) > (1+\delta) m P) \leq \exp\left(-\delta^2 \frac{m P}{3}\right)$$
where $$P = \begin{cases}
\frac{Pr(S \, | \, (x_i, \, x_j))}{m(m-1)} & \text{ in case of sampling pairs}\\
\, & \text{ without replacement}\\
\frac{Pr(S \, | \, (x_i, \, x_j))}{m^2} & \text{ in case of sampling pairs}\\
\, & \text{ with replacement}\\
\end{cases}$$
\end{lem}
In some simplified cases (such as one in the current paper), more informative bounds may be used when the total number $N(S, \vec{x}^{\text{rec}})$ is expected to be small but bigger than $1$.
\begin{lem}\label{expSchemaSmallNumb}
Given a population $\vec{x} = (x_1, \, x_2, \ldots x_m)$ and a schema $S$, suppose $\exists$ a schema $S_0$ such that $\forall \, i$ and $j \in \{1, \, 2, \ldots m\}$, as long as $x_i \in S_0$, $Pr(S, \, | x_i, \, x_j) \geq \alpha$. Then $Pr(N(S, \vec{x}^{\text{rec}}) \geq 1) \geq \left(1 - \left(1 - \frac{N(S_0, \vec{x}^{\text{rec}})}{m}\right)^m\right) \alpha$.
\end{lem}
\begin{proof}
According to the assumption, a sufficient condition to obtain an individual fitting the schema $S$ upon completion of recombination is to select an individual fitting the schema $S_0$ to be the first one in a recombination pair at least once after $m$ consecutive trials and, afterwards, to apply an appropriate recombination transformation with probability at least as large as $\alpha$. An individual fitting the schema $S_0$ is selected at least once with probability $1-\left(1 - \frac{N(S_0, \vec{x}^{\text{rec}})}{m}\right)^m$ via considering the complementary event implying the desired conclusion.
\end{proof}

We now proceed to analyze the context of the next intermediate population $\vec{x}^{\text{mut}}$ obtained from the population $$\vec{x}^{\text{rec}} = (\hat{x}_1, \, \hat{x}_2, \ldots, \hat{x}_m)$$ upon completion of mutation. When applying mutation operator to an individual in position $i$ (i.e. to $\hat{x}_i$), in order to obtain an individual from the set $S$ in the $i^{\text{th}}$ position of the population $\vec{x}^{\text{mut}}$, we must select a mutation operator that sends the individual $\hat{x}_i$ to an element of the set $S$. This event happens with probability
\begin{equation}\label{probAfterMutEqIndiv}
 Pr(S \, | \, x_i) =\sum_{w \in J}\sum_{M \in \mathcal{M}_w}^{M(x_i) \in S}Pr^{\text{MutFamily}}_{(\vec{x}, \, t), \, x_i}(i = w) \cdot Pr^{\text{Mut}, \, w}_{(\vec{x}, \, t), \, x_i}(M).
\end{equation}
The indices of the individuals in the population $\vec{x}^{\text{rec}}$ can be partitioned into two disjoint subsets:
\begin{equation}\label{subsetIntersectPopEq1}
S \cap \vec{x}^{\text{rec}} = \{i \, | \, \hat{x}_i \in S\} \text{ and } \overline{S} \cap \vec{x}^{\text{rec}} = \{i \, | \, \hat{x}_i \notin S\}.
\end{equation}
We now introduce the following probabilities:
\begin{equation}\label{transProbLowerBoundPreserveEq}
\underline{Pr^{mut}(S \, | \, S)} = \min\{Pr(S \, | \, x_i) \, | \, i \in S \cap \vec{x}^{\text{rec}}\}
\end{equation}
to be the minimal probability of preserving the $i^\text{th}$ individual that is already in the set $S$ upon completion of mutation and
\begin{equation}\label{transProbLowerBoundCreateEq}
\underline{Pr^{mut}(S \, | \, \overline{S})}= \min \{Pr(S \, | \, x_i) \, | \, i \in S \cap \vec{x}^{\text{rec}}\}
\end{equation}
to be the minimal probability of mutating the $i^\text{th}$ individual that is not in the set $S$ into one that is in $S$.
Notice that the random variable $N(S, \, \vec{x}^{\text{mut}})$ measuring the total number of individuals in the population $\vec{x}^{\text{mut}}$ from the set $S$ (recall that this random variable has been introduced in the statement of lemma~\ref{lemmaChernoff}) is the sum of independent indicator random variables
$$\mathcal{X}_i = \begin{cases}
1 & \text{if the }i^{\text{th}} \text{ individual of } \vec{x}^{\text{mut}} \in S\\
0 & \text{otherwise:}
\end{cases}$$
\begin{equation}\label{decomposeToIndicatorsEq}
N(S, \, \vec{x}^{\text{mut}}) = \sum_{i=1}^m \mathcal{X}_i = \sum_{i \in S \cap \vec{x}^{\text{rec}}}\mathcal{X}_i + \sum_{i \in \overline{S} \cap \vec{x}^{\text{rec}}}\mathcal{X}_i.
\end{equation}
From the discussion preceding equation~\ref{probAfterMutEqIndiv}, 
\begin{align*}
E(\mathcal{X}_i) = Pr(\mathcal{X}_i = 1) = Pr(S \, | \, x_i)
\end{align*}
 so that, by linearity of expectation, we have
\begin{align*}
E(N(S, \, \vec{x}^{\text{mut}})) 
&= \sum_{i=1}^m E(\mathcal{X}_i) = \sum_{i=1}^m Pr(S \, | \, x_i)
\\&=\sum_{i \in S \cap \vec{x}^{\text{rec}}}Pr(S \, | \, x_i) + \sum_{i \in \overline{S} \cap \vec{x}^{\text{rec}}}Pr(S \, | \, x_i) \\
&\geq  \underline{Pr^{mut}(S \, | \, S)} \cdot |S \cap \vec{x}^{\text{rec}}| + \underline{Pr^{mut}(S \, | \, \overline{S})} \cdot |\overline{S} \cap \vec{x}^{\text{rec}}| 
\\&=\underline{Pr^{mut}(S \, | \, S)} \cdot |S \cap \vec{x}^{\text{rec}}| + \underline{Pr^{mut}(S \, | \, \overline{S})} \cdot \left(m - |S \cap \vec{x}^{\text{rec}}|\right). 
\end{align*}

In summary, we have deduced that if $$\mu = \underline{Pr^{mut}(S \, | \, S)} \cdot |S \cap \vec{x}^{\text{rec}}| + \underline{Pr^{mut}(S \, | \, \overline{S})} \cdot \left(m - |S \cap \vec{x}^{\text{rec}}|\right)$$ then
\begin{equation}\label{estimateExpOfSumEq}
\mu \leq E(N(S, \, \vec{x}^{\text{mut}})).
\end{equation}
The classical Chernoff bound applies again now and tells us that $\forall$ $\delta \in [0, \, 1]$
\begin{align}
Pr\left(N(S, \, \vec{x}^{\text{mut}}) < (1 - \delta)\mu \right) &\overset{\text{thanks to inequality~\ref{estimateExpOfSumEq}}}{\leq}  Pr\left(N(S, \, \vec{x}^{\text{mut}}) < (1 - \delta) \cdot E(N(S, \, \vec{x}^{\text{mut}})) \right) \nonumber\\
\label{ChernoffBoundEq2}
&\leq \exp\left(-\frac{\delta^2}{2} \cdot E(N(S, \, \vec{x}^{\text{mut}})) \right) \leq \exp\left(-\frac{\delta^2}{2}\mu \right).
\end{align}
 
Observe that $|S \cap \vec{x}^{\text{rec}}| = N(S, \, \vec{x}^{\text{rec}})$ (see lemma~\ref{lemmaChernoff}), so that, according to lemma~\ref{lemmaChernoff}, we can bound $\mu$ (see equation-definition preceding equation~\ref{estimateExpOfSumEq}) below as follows: $\forall \, \epsilon \in [0, \, 1]$
\begin{align}
  Pr(\mu \geq \underline{Pr^{mut}(S \, | \, S)} \cdot (1-\epsilon) m P + + \underline{Pr^{mut}(S \, | \, \overline{S})} \cdot \left(m - (1+\epsilon) m P\right)) 
 \label{boundOnDecompEq}
\geq  1 - \exp\left(-\epsilon^2 \frac{m P}{2}\right) - \exp\left(-\epsilon^2 \frac{m P}{3}\right) 
\end{align}
where $P$ is the average probability of obtaining an element in the set $S$ upon completion of recombination as introduced in the statement of lemma~\ref{lemmaChernoff}. Combining inequalities~\ref{ChernoffBoundEq2} and \ref{boundOnDecompEq} we finally deduce the following lower bound on the probability of the number of occurrences of individuals from the set $S$ occurring in the population $\vec{x}^{\text{mut}}$:
Let  
\begin{equation}\label{defOfSecondLemmaBoundEq}
\underline{\mu} =   m \left(\underline{Pr^{mut}(S \, | \, S)} \cdot (1-\epsilon)P + \underline{Pr^{mut}(S \, | \, \overline{S})} \left(1 - P(1+\epsilon)\right)\right)
\end{equation}
Then we have 
\begin{align}
Pr\left(N(S, \, \vec{x}^{\text{mut}}) \geq (1 - \delta)\underline{\mu} \right)     &\geq Pr\left(N(S, \, \vec{x}^{\text{mut}}) \nonumber \geq (1 - \delta)\mu \, | \, \mu \geq \underline{\mu} \right) \cdot Pr(\mu \geq \underline{\mu}) \nonumber \\
 &
\overset{\text{via inequalities~\ref{ChernoffBoundEq2} and \ref{boundOnDecompEq}}}{\geq} \left(1 - \exp\left(-\frac{\delta^2}{2}\underline{\mu} \right) \right) 
\label{secondLemmaBoundEq}
\times \left( 1 - \exp\left(-\epsilon^2 \frac{m P}{2}\right) - \exp\left(-\epsilon^2 \frac{m P}{3}\right)\right).
\end{align} 
We summarize ineqaulity~\ref{secondLemmaBoundEq} in the following lemma.
\begin{lem}\label{afterMutationLemma}
Given a pair $(\vec{x}, \, t)$ where $\vec{x}$ is a population at $t^{\text{th}}$ generation of an EA and any subset $S \in \Omega$, continuing with the notation in lemmas~\ref{prodProbLem} and \ref{lemmaChernoff}, as well as equation-definitions~\ref{transProbLowerBoundPreserveEq}, \ref{transProbLowerBoundCreateEq} and \ref{defOfSecondLemmaBoundEq}, select a pair of small numbers $(\delta, \, \epsilon) \in [0, \, 1]^2$. Then the probability that the total number of individuals in the ``intermediate" population obtained from the population $\vec{x}$ upon completion of recombination followed by mutation is above the threshold $(1 - \delta)\underline{\mu}$ is at least $$\left(1 - \exp\left(-\frac{\delta^2}{2}\underline{\mu} \right) \right) \times  \left( 1 - \exp\left(-\epsilon^2 \frac{m P}{2}\right) - \exp\left(-\epsilon^2 \frac{m P}{3}\right)\right).$$
\end{lem}
Once again, in simplified constructions such as one presented in the current preliminary work, the following alternative lemma is an immediate corollary of lemma~\ref{expSchemaSmallNumb}:
\begin{lem}\label{afterMutationSimplifiedLem}
Continuing with the assumptions of lemma~\ref{expSchemaSmallNumb}, suppose, in addition, that $\forall \, x_i \in S$ the probability $Pr(S \, | \, x_i) \geq \beta$. Then $$Pr(N(S, \vec{x}^{\text{mut}}) \geq 1) \geq \left(1 - \left(1 - \frac{N(S_0, \vec{x}^{\text{rec}})}{m}\right)^m\right) \alpha \cdot \beta.$$
\end{lem}

The following generalized schema theorem is nearly a restatement of lemma~\ref{afterMutationLemma} that takes into account generalized selection as described in the previous section.
\begin{thm}\label{genSchemaThm}
Let $\vec{z}$ denote the population obtained from the population $\vec{x}$ upon completion of the recombination $\rightarrow$ mutation $\rightarrow$ selection cycle (equivalently, the population $\vec{z}$ is obtained from the populations $\vec{x}$ and $\vec{x}^{\text{mut}}$ after selection). Recall from the statement of lemma~\ref{lemmaChernoff} that the random variable $N(S, \, \vec{z})$ counts the total number of individuals from the set $S$ that appear in the population $\vec{z}$. Repeat verbatim the first sentence of lemma~\ref{afterMutationLemma}.  Then $\forall \, n \in \{1, \, 2, \ldots, m\}$
\begin{align*}
&Pr(N(S, \, \vec{z}) \geq n)  \\
 \geq &Pr\left(N(S, \, \vec{z}) \geq n \, | \, N(S, \, \vec{x}^{\text{mut}}) \geq (1 - \delta)\underline{\mu}\right)   \left(1 - \exp\left(-\frac{\delta^2}{2}\underline{\mu} \right) \right)   \left( 1 - \exp\left(-\epsilon^2 \frac{m P}{2}\right) - \exp\left(-\epsilon^2 \frac{m P}{3}\right)\right).
\end{align*}
 
\end{thm}
Applying the classical Markov inequality (see, for instance, \cite{DoerrB}), we immediately deduce the following.
\begin{cor}\label{genSchemaThmExp}
Continuing with the notation and assumptions in theorem~\ref{genSchemaThm}, $\forall$ real $k \in \left[0, \, \frac{m}{(1 - \delta)\underline{\mu}}\right]$ 
\begin{align*}
E(N(S, \, \vec{z})) \geq   &
\lceil k(1 - \delta)\underline{\mu} \rceil \cdot \left(1 - \exp\left(-\frac{\delta^2}{2}\underline{\mu} \right) \right)  \times \left( 1 - \exp\left(-\epsilon^2 \frac{m P}{2}\right) - \exp\left(-\epsilon^2 \frac{m P}{3}\right)\right) \\
 \times &Pr\left(N(S, \, \vec{z})  
 \geq  \lceil k(1 - \delta)\underline{\mu} \rceil \, | \, N(S, \, \vec{x}^{\text{mut}}) 
  \geq (1 - \delta)\underline{\mu} \right).
\end{align*}
\end{cor}
The corresponding simplified schema theorem is a direct consequence of lemma~\ref{afterMutationSimplifiedLem}:
\begin{cor}\label{genSchemaTheoremSimple}
Continuing with the notation and assumptions in theorem~\ref{genSchemaThm}, $$Pr(N(S, \, \vec{z}) \geq 1) \geq \alpha \beta Pr\left(N(S, \, \vec{z}) \geq 1 \, | \, N(S, \, \vec{x}^{\text{mut}}) 
\\ \geq 1\right)   \left(1 - \left(1 - \frac{N(S_0, \vec{x}^{\text{rec}})}{m}\right)^m\right).$$
\end{cor}
\section{Recombination-invariant  Subsets, Fitness Levels and  Mutation-invariant Subsets}\label{RecInvarAndDriftSect}
While theorem~\ref{genSchemaThm} and corollary~\ref{genSchemaThmExp} are valid for arbitrary subsets $S \subseteq \Omega$, it is not in vain that most notions of schemata (see, for instance, \cite{Antonisse} and \cite{PoliSchema}) happen to be \emph{recombination-invariant} subsets of the search space $\Omega$ as defined precisely below:
\begin{defn}\label{schemaDefn}
Given a family of recombination transformations $\mathcal{F}$ on a search space $\Omega$, a \emph{recombination invariant subset} or, alternatively, a \emph{generalized schema} with respect to the family of recombination transformations $\mathcal{F}$ is a subset $H \subseteq \Omega$ having the property that $\forall \, x \text{ and } y \in H$ and $\forall$ transformation $T \in \mathcal{F}$ the child $T(x, \, y) \in H$.
\end{defn}
General mathematical properties of recombination-invariant subsets have been studied by several authors: see, for instance, \cite{Radcliffe}, \cite{Stadlers1}, \cite{Stadlers2}, \cite{MitavFirst}, \cite{MitavCompair} and \cite{MitavCatCompair}. First of all, we list a few basic properties of families of recombination-invariant subsets (see \cite{MitavFirst} for a detailed exposition and an in-depth analysis of the relationship between the collections of recombination-invariant subsets of the search space and the corresponding families of recombination transformations).
\begin{prop}\label{basicPropOfInvar1}
Given any family of recombination transformations $\mathcal{F}$ on the search space $\Omega$, the corresponding family of invariant subsets with respect to the family $\mathcal{F}$, call it $Sel_{\mathcal{F}}$, is closed under arbitrary intersections, contains the $\emptyset$ and the whole search space $\Omega$\footnote{such collections of subsets are also known as pre-topologies: see \cite{Stadlers1} and \cite{Stadlers2}}. Furthermore, given any collection of subsets $\mathcal{S} \subseteq \mathcal{P}(\Omega)$ of the search space $\Omega$, the family $\underline{S} = \{\bigcap_{S \in \mathcal{T}} \, | \, \mathcal{T} \subseteq \mathcal{S}\} \cup \{\emptyset, \, \Omega\}$ is closed under arbitrary intersections, contains the $\emptyset$ and the entire space, and $\exists$ a family of recombination transformations $\mathcal{F}$ such that $\mathcal{S}_{\mathcal{F}} = \underline{S}$. Consequently, the union of all families of recombination transformations with the above property is the unique maximal (in the sense of containment) family of recombination transformations, call it $\widetilde{\mathcal{F}}$, such that $\mathcal{S}_{\widetilde{\mathcal{F}}} = \underline{S}$. We will say that the collection of recombination-invariant subsets $\underline{S}$ is generated by the collection of subsets $\mathcal{S}$ or, alternatively, that the collection of subsets $\mathcal{S}$ generates the collection of recombination invariant subsets of the search space $\underline{S}$.
\end{prop}
The correspondence summarized in proposition~\ref{basicPropOfInvar1} is known in mathematics as a Galois connection\footnote{see \cite{BarrWells} for the notions of natural transformations, adjunctions and Galois connections i.e. natural transformations between posets considered as categories. No knowledge of category theory is necessary to understand the current paper though.}. One of the central ideas of the current article is that recombination transformations should be designed based on the suitable families of recombination-invariant subsets and below we will suggest how such families of recombination-invariant subsets may be selected to design efficient algorithms. This design is largely based on the notion of a fitness level introduced below.
\begin{defn}\label{fitnessLevelDefn}
Given a fitness function $f: \Omega \rightarrow [0, M]$, the $k^{\text{th}}$ \emph{fitness level} of $f$ is the pre-image $$f^{-1}([k, \, M]) = \{\omega \in \Omega \, | \, f(\omega) \geq k\}.$$
\end{defn}
Recall from section~\ref{MathDescribeSect} that hybrid and mixed strategy EAs may have a large number of auxiliary fitness functions. The auxiliary fitness functions are often defined implicitly in terms of a certain incremental deterministic algorithm to find a satisfactory solution for a specific NP-hard optimization problem. We impose the following conditions on our Hybrid or mixed strategy EA:
%$\,$

\textbf{Condition 1.} The total number of auxiliary fitness functions is bounded above by a polynomial of degree $\rho$ in the size of the problem instance (in other wards, $L = O(n^{\rho})$ where $L$ is the indexing set of the auxiliary fitness functions as in section~\ref{MathDescribeSect} and $n$ is the size of the problem instance.

\textbf{Condition 2.} All of the auxiliary fitness functions are non-negative, integer valued\footnote{This assumption does not reduce the generality since there are finitely many auxiliary fitness functions and one can always ``shift all of them up" by an additive positive constant.} and have a common range\footnote{The assumption of having a common range can be alleviated at the cost of technical complications that divert attention away from the mainstream idea of the current article.} and there are polynomially many auxiliary fitness levels: $\forall \, l \in L$ $f_l: \Omega \rightarrow \{0, \, 1, \ldots, M\}$ and $M \leq O(n^{\tau})$ where $n$ is the size of an instance of a specific optimization problem.

\textbf{Condition 3.} $\exists \, l \in L$ such that the $M^{\text{th}}$ fitness level of the auxiliary fitness function $f_l$ consists of ``satisfactory solutions" (for instance, up to a specified approximation ratio) for the objective fitness function $f_0$.

There is a number of ways to design the collections of recombination invariant subsets based on the fitness levels of various auxiliary fitness functions to guarantee that the expected time (i.e. the expected number of iterations) an EA requires to encounter a satisfactory solution is polynomial in the size of the input instance. Since the aim of the current paper is to illustrate the general ideas for such designs, we present what is, perhaps, one of the simplest and the shortest methodologies. For $k \in \{1, \, 2, \ldots M\}$, let $S_k = \bigcup_{l \in L} f_l^{-1}([k, \, M])$ and let $\mathcal{S}_k = \{S_j \, | j \geq k\} \cup \{\emptyset, \, \Omega\}$. Observe that the collection $\mathcal{S}_k$ of the unions of fitness levels at least as high as $k$ is a collection of nested sets so that, in particular, it is closed under arbitrary intersections. According to proposition~\ref{basicPropOfInvar1} we may select families of recombination transformations $\mathcal{F}_k$ such that the corresponding families of invariant subsets $\mathcal{S}_{\mathcal{F}_k} = \mathcal{S}_k$. In fact, all that we require is that $\mathcal{S}_{\mathcal{F}_k} \supseteq \mathcal{S}_k$ i.e. that the family of recombination transformations $\mathcal{F}_k$ preserves the unions of $l^{\text{th}}$ fitness levels across all of the auxiliary fitness functions.

We now turn our attention to mutation transformations. Invariant subsets for mutation transformations are defined in the same fashion.
\begin{defn}\label{mutationInvarSubsets}
Given a family $\mathcal{M}$ of mutation transformations, a subset $S \in \Omega$ is \emph{invariant under the family of mutation transformations} $\mathcal{M}$ if $\forall \, x \in S$ and $\forall \, M \in \mathcal{M}$ the individual $M(x) \in S$ as well. We write $\mathcal{S}_{\mathcal{M}}$ to denote the collection of all subsets that are invariant under the family of mutation transformations $\mathcal{M}$.
\end{defn}
Families of mutation-invariant subsets enjoy the same properties as these of recombination-invariant subsets as described in proposition~\ref{basicPropOfInvar1}.\footnote{In fact, this applies to arbitrary families of $m$-ary transformations on $\Omega$ and their corresponding families of invariant subsets.} Once again, we design the families of mutation transformations based on preferable family of mutation-invariant subsets. Just as the case with recombination, a vast number of designs are possible, yet, for illustrative purposes, we select one of the simplest in the current article. We let the indexing family of our hybrid or mixed strategy EA $J = \{1, \, 2, \ldots M\}$ and for each $j \in J$ we select a family of mutation transformations $\mathcal{M}_j$ the collection of mutation-invariant subsets of which is, just as in case of recombination, $\mathcal{S}_j = \{S_q \, | q \geq j\} \cup \{\emptyset, \, \Omega\}$ with $S_q = \bigcup_{l \in L} f_l^{-1}([q, \, M])$. There is a further requirement on the families of mutation transformations though:

\textbf{Condition 4.} We require that the family of mutation transformations $\mathcal{M}_q$ for $q = M$ (i.e. at the highest common auxiliary fitness level) possesses the following property: whenever $l$ and $k \leq L$, $\forall \, x \in f_l^{-1}(\{M\})$ $\exists$ $y \in f_k^{-1}(\{M\})$ and a sequence of mutation transformations $T_1, \, T_2, \ldots, T_i \in \mathcal{M}_l$ such that $y = T_i \circ T_{i-1} \circ \ldots \circ T_1 (x)$. We further require that $\exists$ polynomials $n^{\gamma}$ and $n^{\lambda}$ such that $\forall \, l$ and $k \leq L$ and $\forall \, x \in f_l^{-1}(\{M\})$ $\exists$ $y \in f_k^{-1}(\{M\})$ such that the probability that $y$ is encountered after $O(n^{\gamma})$ applications of the mutation transformations from the families $\mathcal{M}_l$ is at least $\Omega(n^{-\lambda})$ regardless of the population in which the individual $x$ appears and the iteration time at which the population arises.

The following simple lemma hints at the motivation for condition 4 in our design.
\begin{lem}\label{expectLastMutationLemma}
Consider any Markov chain on the state space $S_M = \bigcup_{l \in L} f_l^{-1}(\{M\})$ with the transition matrix \\$\{p_{x \rightarrow y}\}_{x, \, y \in S_M}$ defined as $$p_{x \rightarrow y} = \mu_x\{F \, | \, F \in \mathcal{M}_M \text{ and } F(x) = y\}$$ where $\{\mu_x\}_{x \in S_M}$ is the collection of probability measures on $\mathcal{M}_M$ satisfying condition 4 above. For an $x \in S_M$ let $T_x$ denote the random waiting time to encounter a ``satisfactory" solution with respect to the objective fitness function $f_0$ for the first time. Then $\forall \, x \in S_M$ $E(T_x) \leq O\left(n^{\gamma + \lambda}\right)$.
\end{lem}
\begin{proof}
According to condition 3, $\exists \, l_0 \in L$ such that any $y \in f_{l_0}^{-1}(\{M\})$ is a satisfactory ``satisfactory" solution with respect to the objective fitness function $f_0$ so that for any given $x \in S_M$ the random variable $T_x$ is bounded above by the random time $T_{x \rightarrow l_0}$ of encountering an individual $y \in f_{l_0}^{-1}(\{M\})$ for the first time. Comparing the random time $T_{x \rightarrow l_0}$ with a geometric random variable $T$ with a unit step size $O(n^{\gamma})$ and success probability $\Omega(n^{-\lambda})$, thanks to condition 4 we deduce that $E(T_{x \rightarrow l_0}) \leq n^{\gamma} E(T) = O(n^{\gamma}) \cdot \frac{1}{\Omega(n^{-\lambda})} = O\left(n^{\gamma + \lambda}\right)$ as claimed.
\end{proof}
The type of hybrid EA's design suggested in the current article is largely based on the notion of an individual's maximal auxiliary fitness level introduced below:
\begin{defn}\label{auxFitnessLevelDef}
For an individual $x \in \Omega$ the \emph{maximal auxiliary fitness level} of $x$ is $auxFit(x) = \max_{l \in L} f_l(x)$.
\end{defn}
In other words, the maximal auxiliary fitness level of an individual $x$ is the largest auxiliary fitness level $q$ of $x$ so that $x$ ``fits" the schema $S_q$ that is invariant under the family of recombination transformations $\mathcal{F}_j$ for $j \geq q$. A rather simple complexity analysis presented in the current paper relies on applying the simplified version of schema theory from section~\ref{schemaSection} to the special schemata introduced in the following definition:
\begin{defn}\label{MaximalAuxFitnessSchemaDefn}
Consider a hybrid or mixed strategy EA fitting the framework of the current article that uses populations of size $m \geq 2$, and a population $\vec{x} = \{x_1, \, x_2, \ldots x_m\}$. Let $AuxMax(\vec{x}) = \max \{auxFit(x_i) \, | \, 1 \leq i \leq m\}$ denote the maximal auxiliary fitness level present in the population $\vec{x}$. We say that the  schema $S_{AuxMax(\vec{x})}$ is the \emph{leading current schema} while the schema $S_{AuxMax(\vec{x})+1}$ is the \emph{leading future schema} of the population $\vec{x}$.
\end{defn}
One of the simplest (but not the only possible) designs of ``efficient" hybrid or mixed strategy EAs is to concentrate the probability distributions $Pr^{\text{RecFamily}}_{(\vec{x}, \, t), \, (x_i, \, x_j)}$ on the set indexing the families of recombination transformations (recall that there are as many of these as there are auxiliary fitness levels) on the indices $q \geq auxFit(x_i)$ whenever $auxFit(x_i) = AuxMax(\vec{x})$. Likewise, the probability of selecting the families of mutation transformations, $Pr^{\text{MutFamily}}_{(\vec{x}, \, t), \, x_i}$, is also concentrated on the indices $q \geq auxFit(x_i)$ for the individuals $x_i$ fitting the leading current schema. For $q \in \{1, \, 2 ,\ldots M\}$ let
$$P_{Imp}^{Rec}(q)  
 =\min\{Pr^{\text{RecFamily}}_{(\vec{x}, \, t), \, (x_i, \, x_j)}(w > q) \, | \, auxFit(x_i)=q, \, x_j \in \Omega\}$$
and
\begin{equation}\label{minImprovProbEq}
P_{Imp}^{Mut}(q) = \min\{Pr^{\text{MutFamily}}_{(\vec{x}, \, t), \, x_i}(w \geq q) \, | \, auxFit(x_i)=q\}.
\end{equation}
denote the minimal probabilities of improving the auxiliary fitness level $q$ after applying recombination and mutation transformations respectively.
We assume that all the probabilities of types $P_{Imp}^{Rec}(q)$ and $P_{Imp}^{Mut}(q)$ are positive.

In the current paper we do not assume much about selection apart from preserving the highest auxiliary as well as the highest objective fitness levels. Formally this can be defined as follows.
\begin{defn}\label{elitystHybridSelDefn}
We say that a selection transformation $Sel: (\Omega^m)^2 \rightarrow \Omega^m$ is hybrid-elitist if $\forall \vec{x}$ and $\vec{y} \in \Omega^m$ $$AuxMax(Sel(\vec{x}, \, \vec{y})) = \max\{AuxMax(\vec{x}), \, AuxMax(\vec{y})\}$$ and $$\max\{f_0(Sel(\vec{x}, \, \vec{y})) \, | \, 1 \leq i \leq m\} =  \max\{\max\{f_0(x_i) \, | \, 1 \leq i \leq m\}, \, \max\{f_0(y_i) \, | \, 1 \leq i \leq m\}\}$$
\end{defn}

Applying corollary~\ref{genSchemaTheoremSimple} with $S_0$ being the leading current schema and $S$ being the leading future schema as in definition~\ref{MaximalAuxFitnessSchemaDefn} and observing that $N(S_{AuxMax(\vec{x})}, \vec{x}^{\text{rec}}) \geq 1$ by definition~\ref{MaximalAuxFitnessSchemaDefn}, we immediately deduce the following fact.
\begin{lem}\label{probOfImprovLemma}
Suppose we are given a hybrid or mixed strategy EA with constant population size $m$ that exploits hybrid-elitist selection. Then 
\begin{align*}
 Pr(N(S_{AuxMax(\vec{x})+1}, \vec{z}) \geq 1) \geq&\left(1 - \left(1 - \frac{1}{m}\right)^m\right)\times
  P_{Imp}^{Rec}(AuxMax(\vec{x})) \cdot P_{Imp}^{Mut}(AuxMax(\vec{x})+1)\\
>&(1 - \exp(-1))\cdot P_{Imp}^{Rec}(AuxMax(\vec{x})) \cdot P_{Imp}^{Mut}(AuxMax(\vec{x})+1)
\end{align*} where, as usual, $\vec{z}$ denotes the population obtained from the population $\vec{x}$ upon completion of a single recombination $\longrightarrow$ mutation $\longrightarrow$ selection cycle.
\end{lem}
\section{Drift Analysis and Simple Runtime Complexity Bounds}\label{mainResultsSect}
Drift analysis methodology invented in \cite{HajekDrift} has been introduced into the evolutionary computation theory for estimating the expected run-time complexity in \cite{he2001drift}, quickly gained popularity and has been modified and enhanced in a number of ways. In the current article we will use the additive variable drift analysis version established in \cite{MitavRowCanningsNetworksDrift}. For the sake of completeness, the necessary definition and the relevant lemma are stated below.
\begin{defn}\label{distFunctDefn}
Let $(\mathcal{X}, \, \{p_{x \rightarrow y}\}_{x, \, y \in
\mathcal{X}})$ denote a Markov chain with finite state space
$\mathcal{X}$ and transition probabilities $p_{x \rightarrow y}$ for
$x$ and $y \in \mathcal{X}$. Let $A \subseteq \mathcal{X}$. A
distance function $D$ on $\mathcal{X}$ with respect to $A$ is any
function $D: \mathcal{X} \rightarrow [0, \infty)$ with the property
that $D(x) = 0$ if and only if $x \in A$. Let
$\{X_t\}_{t=0}^{\infty}$ denote the stochastic process associated
with the Markov chain $\mathcal{X}$. We are interested in the
following waiting time random variable:
$$T(\mathbf{x} \, | \, X_0 = \xi_0) = \min\{t \, | \, X_t(\mathbf{x})
\in A\}$$ under the assumption that $X_0(\mathbf{x}) = \xi_0$ with
probability $1$ (i.e. the chain starts at a specified $\xi_0 \in
\mathcal{X}$).
\end{defn}
A simple complexity bound appearing in the current paper is based on the following
additive variable drift lemma from \cite{MitavRowCanningsNetworksDrift}.
\begin{lem}\label{timeUpperBound}
Suppose we are given a Markov chain\\ $(\mathcal{X}, \, \{p_{x
\rightarrow y}\}_{x, \, y \in \mathcal{X}})$, a subset $A \subseteq
\mathcal{X}$ and a distance function $D: \mathcal{X} \rightarrow [0,
\infty)$ as described in definition~\ref{distFunctDefn}. Suppose
also that for every integer $k \in \mathbb{N} \cup \{0\}$ $\exists$
a constant $l_k \in (0, \infty)$ such that $\forall \, x \in A^c$
with $\lceil D(x) \rceil \geq k$ (here $A^c$ denotes the complement
of $A$ in $\mathcal{X}$) we have $D(x)-\sum_{y \in \mathcal{X}}p_{x
\rightarrow y}D(y) \geq l_k$. Then
$$E(T(\mathbf{x} \, | \, X_0 = \xi_0)) \leq \sum_{k=1}^{\lceil D(\xi_0)
\rceil} \frac{1}{l_k}.$$
\end{lem}
Given a hybrid or mixed strategy $EA$ with the highest auxiliary fitness level $M$ satisfying conditions $1$, $2$ and $3$ that exploits hybrid-elitist selection, we apply lemma~\ref{timeUpperBound} to the Markov chain $\mathcal{X}$ of all populations of size $m$, $\Omega^m$ with the probability $p_{\vec{x} \rightarrow \vec{z}}$ being the probability that the population $\vec{z}$ is obtained from the population $\vec{x}$ upon completion of a recombination $\longrightarrow$ mutation $\longrightarrow$ selection cycle. The set $$A = \{\vec{x} \, | \, \vec{x} \in \Omega^m \text{ and } AuxMax(\vec{x}) = M\}$$ is the set of all populations containing an individual of the highest auxiliary fitness level and the distance function $D: \mathcal{X} \rightarrow \{0, \, 1, \ldots, M\}$ defined as $D(\vec{x}) = M - AuxMax(\vec{x})$. According to lemma~\ref{probOfImprovLemma}, whenever $AuxMax(\vec{x}) = q$ 
\begin{align*}
Pr(D(\vec{x}) - D(\vec{z}) \geq 1 \, | \, p_{\vec{x} \rightarrow \vec{z}} \neq 0)  
\geq&\left(1 - \left(1 - \frac{1}{m}\right)^m\right)\cdot P_{Imp}^{Rec}(q) \cdot P_{Imp}^{Mut}(q+1)\\
>& (1-\exp(-1))P_{Imp}^{Rec}(q) \cdot P_{Imp}^{Mut}(q+1)\end{align*}
 Furthermore, thanks to the assumption that our EA exploits a hybrid-elitist selection, whenever $p_{\vec{x} \rightarrow \vec{z}} \neq 0$ $D(\vec{x}) - D(\vec{z}) \geq 0$. Now, letting
\begin{equation}\label{improvSchemaLowerBoundsVar}
l_k(m) = \left(1 - \left(1 - \frac{1}{m}\right)^m\right)\cdot P_{Imp}^{Rec}(k) \cdot P_{Imp}^{Mut}(k+1)
\end{equation}
and
\begin{equation}\label{improvSchemaLowerBoundsConst}
l_k = (1-\exp(-1))\cdot P_{Imp}^{Rec}(k) \cdot P_{Imp}^{Mut}(k+1)
\end{equation}
we deduce that the worst case expected runtime complexity upper bound to reach the highest auxiliary fitness level for a hybrid or mixed strategy EA that exploits hybrid-elitist selection (and hence must have a population size $m \geq 2$) is $\sum_{k=0}^{M-1} \frac{1}{l_k(m)} \leq \sum_{k=0}^{M-1} \frac{1}{l_k}$. If we assume, in addition, that our EA satisfies condition $2$ in section~\ref{RecInvarAndDriftSect} then $M = O(n^{\tau})$. Furthermore, conditions $1$, $3$ and $4$ allow us to apply lemma~\ref{expectLastMutationLemma} and to deduce the following expected runtime result.
\begin{thm}\label{expRunTimeSimpleFinalRes}
Suppose a given hybrid or mixed strategy EA that employs hybrid-elitist selection satisfies conditions $1$, $2$, $3$ and $4$ in section~\ref{RecInvarAndDriftSect}. Then the worst-case expected runtime for the EA to reach a satisfactory  solution is bounded above by $$\sum_{k=0}^{M-1} \frac{1}{l_k(m)} + n^{\gamma + \lambda} \leq \sum_{k=0}^{M-1} \frac{1}{l_k} + n^{\gamma + \lambda}.$$
\end{thm}
In the next section we illustrate an application of theorem~\ref{expRunTimeSimpleFinalRes} and, more importantly, the methodology developed in sections~\ref{schemaSection} and \ref{RecInvarAndDriftSect} as well as in the current section with a single machine scheduling problem.
\section{An Example Application: Single Machine Scheduling Problem}\label{exampleSect}
One of many classical NP-hard combinatorial optimization problems is the single machine scheduling problem (see chapter on scheduling problems by Leslie Hall in \cite{LeslieHall}). The instance of the problem of size $n$ consists of a sequence of ordered triples $\{j_i\}_{i=1}^n$ where $j_i = (r_i, \, p_i, \, q_i)$ with $r_i, \, p_i$ and $q_i \in [0, \, \infty)$ standing for ``release time", ``processing time" and ``delivery time" of the job $j_i$ respectively. Each of the jobs has to be processed on a single machine, call it $M$, without interruption and it starts getting delivered immediately after being processed. There is no restriction on the total number of jobs being delivered simultaneously, yet a job can start getting processed no earlier than its release time $r_i$ and only when the machine is available (i.e. not processing another job). The jobs can be processed in any order and the objective is to minimize the ``maximal lateness" of the schedule: i.e. the time instant when the last job has been just delivered. Thus, the search space $\Omega = \{\pi \, | \, \pi: I_n \rightarrow I_n \text{ is a permutation}\}$ where $I_n = \{1, \, 2, \ldots, n\}$ so that every $\pi \in \Omega$ determines the schedule $(j_{\pi(1)}, \, j_{\pi(2)}, \ldots, j_{\pi(n)})$. Let $s_{\pi(i)}$ denote the time when the job $j_{\pi(i)}$ starts processing. Then the \emph{maximal lateness} of the schedule $\pi$ is $J_{\pi} = \max_{1 \leq i \leq n} \{s_{\pi(i)}+p_{\pi(i)}+q_{\pi(i)}\}$. and the objective is to find a permutation $\sigma \in \Omega$ such that $J_{\sigma} = J^{*} = \min\{J_{\pi} \, | \, \pi \in \Omega\}$. Let now $P = \sum_{i=1}^{n}p_i$ and let $\epsilon > 0$ be given. Let $\delta = \epsilon P$ and let $B_{\epsilon} = \{i \, | \, p_i \geq \delta\}$. Then $|B_{\epsilon}| \leq \frac{P}{\epsilon P} = \frac{1}{\epsilon}$. Let $\Phi = \{\phi : B_{\epsilon} \rightarrow I_n \, | \, \phi \text{ is one-to-one}\}$ denote the set of \emph{repositioning maps} and notice that the total number of such repositioning maps is bounded above as $|\Phi| = n^{|B_{\epsilon}|} \leq n^{\frac{1}{\epsilon}}$ (this verifies condition 1 in the upcoming design). The following notion is crucial in determining the auxiliary fitness functions (see \cite{MitavJunSchedule} for a significantly more detailed exposition).
\begin{defn}\label{partialJacksonDefn}
We say that a schedule $\pi \in \Omega$ is $(k, \, \epsilon, \, \phi)$-Jackson if $\forall \, h \leq k$ if $h = \phi(i)$ for some $i \in I_n$ then $\pi(i) = h = \phi(i)$ or else let $a(h-1)$ denote the time when the job $\pi(h-1)$ has just finished being processed and consider the set of all the jobs $A_h = \{u \, | \, u \notin B_{\epsilon} \text{ and }r_u \geq a(h-1)\}$. It is then required that $q_{\pi(h)} = \max\{q_u \, | \, u \in A_h\}$.
\end{defn}
Plainly speaking, the idea behind definition~\ref{partialJacksonDefn} is the stepwise implementation of the ``partial Jackson rule": whenever the machine is available, as long as a long-processing job (with $p_i \geq \delta$) is not one booked to be scheduled at time $h$, then schedule a job with ``short" processing time that has the longest delivery time next. Our set of auxiliary fitness functions is indexed by the family of repositioning maps $\Phi$ and $\forall \; \phi \in \Phi$ and $\pi \in \Omega$, $f_{\phi}(\pi) = \max\{k \, | k \in I_n \text{ and }\pi \text{ is }(k, \, \epsilon, \, \phi)-\text{Jackson}\}$. The following clever result that is implicitly established in \cite{LeslieHall} and enforces condition 3 to hold in our design appears below:
\begin{thm}\label{mainSchedulingThm}
$\exists \, $ a repositioning map $\phi \in \Phi$ such that $f_{\phi}(\pi) = n \Longrightarrow f_0(\pi) = J_{\pi} \leq J^{*}(1 + \epsilon)$.
\end{thm}
Since there are totally $n = n^1$ auxiliary fitness levels, condition 2 is fulfilled automatically. The only remaining part is to design the families of recombination and mutation transformations preserving the ``auxiliary cross-fitness level schemata" and the highest auxiliary fitness level mutation transformations according to the recipe in section~\ref{RecInvarAndDriftSect}. This can be done in a vast number of ways (see \cite{MitavFirst} for a detailed analysis of the relationship between families of recombination-invariant subsets and the families of recombination transformations fixing them). Here is one possibility. For an auxiliary fitness level $i$ we define the family of hybrid recombination transformations $\mathcal{F}_i = \{F_{\zeta}^i \, | \, \zeta \in \mathcal{S}_{I_{n-i}}\}$ where $I_{n-i}$ denotes the indexing set $\{1, \, 2, \ldots, n-i\}$ while $\mathcal{S}_{I_{n-i}}$ denotes the group of all permutations on the set $\mathcal{S}_{I_{n-i}}$ as follows. Select a permutation $\zeta$ on $I_{n-i}$. Given a pair of permutations $(\pi, \, \sigma) \in \Omega^2$ with $\pi = (\pi(1), \, \pi(2), \ldots, \pi(n))$ and $\sigma = (\sigma(1), \, \sigma(2), \ldots, \sigma(n))$, let $F_{\zeta}^i(\pi, \, \sigma) = \eta = (\eta(1), \, \eta(2), \ldots, \eta(n))$ with $\eta(l) = \pi(l)$ whenever $1 \leq l \leq i$. Now extract a subsequence of jobs in $\sigma$ that do not appear among the first $i$ jobs in $\pi$ and notice that there must be exactly $n-i$ such jobs. The ordering of these jobs in $\sigma$ can be represented by a permutation $\omega \in \mathcal{S}_{I_{n-i}}$ and the composition $\zeta \circ \omega$ produces another ordering of these remaining jobs. We schedule them all right after the job $\eta(i)=\pi(i)$ in the schedule $\eta$ in the ordering $\zeta \circ \omega$. Observe that if $auxFit(\pi) = i$ and a transformation $F_{\zeta}^i \in \mathcal{F}_i$ is selected uniformly at random, then $\forall \, \sigma \in \Omega$ the probability that $auxFit(F_{\zeta}^i(\pi, \, \sigma)) > i$ is at least $\frac{1}{n-i}$ since at least one of the $n-i$ jobs that do not appear among the first $i$ jobs in $\pi$, when scheduled after the job $j_{\pi(i)}$ must improve the auxiliary fitness level for at least one of the auxiliary fitness functions. The next step is to design mutation transformations and, once again, the number of ways to do so is countless. Here we present the following very simple design: for an auxiliary fitness level $i$, let $\mathcal{M}_i = \{M_{a, \, b}^i \, | \, a \text{ and }b \in I_n\}$ where $M_{a, \, b}^i(\pi) = \widetilde{\pi}$ and $\widetilde{\pi} = \pi$ unless $\pi(a)$ or $\pi(b) \in B_{\epsilon}$ in which case the positions of these jobs are swapped and then, if at least one of the jobs that has been swapped appears below the $i^{\text{th}}$ position, the partial Jackson rule with respect to the new positioning $\widetilde{\phi}$ of the jobs in $B_{\epsilon}$ is applied starting with the lowest index of one of the repositioned job up to the $i^{\text{th}}$ fitness level of the auxiliary fitness function $f_{\widetilde{\phi}}$ thereby obtaining a new schedule $\widetilde{\pi}$. It follows then that $f_{\widetilde{\phi}}(\pi) \geq i = auxFit(\pi)$. We equip each family of mutation transformations $\mathcal{M}_i$ with the uniform probability distribution. To apply theorem~\ref{expRunTimeSimpleFinalRes}, all that remains now is to check condition 4 in section~\ref{RecInvarAndDriftSect}. Here we use the classical fact ``about card shuffling via random transpositions", the simplest analysis of which is presented as an elegant illustration of the Markov chain coupling methodology in Chapter 4-3, section 1.7 of \cite{AldousD}, it easily follows that if we are given a schedule $\pi$ with $auxFit(\pi) = n$ (the highest auxiliary fitness level) and another schedule $\sigma$ with $auxFit(\sigma) = n$, after $O(n^2)$ time steps, the probability that  the schedule $\sigma$ has been encountered after repeated application of the mutation transformations from the family $\mathcal{M}_n$ is at least $\Omega(n^{\frac{1}{\epsilon}})$, thereby establishing the desired condition 4. We are now in a position to apply theorem~\ref{expRunTimeSimpleFinalRes} to deduce that the expected runtime until encountering a population containing a schedule $\pi$ with $f_0(\pi) = J_{\pi} \leq J^{*}(1 + \epsilon)$ is no bigger than $\sum_{i=0}^{n-1}\frac{1}{n-i} + O(n^{2 + \frac{1}{\epsilon}}) = \Theta(\ln(n)) + O(n^{2 + \frac{1}{\epsilon}}) = O(n^{2 + \frac{1}{\epsilon}})$.

\section{Conclusion}
While classical schema theory has been widely criticized in the setting of traditional EAs: see, for instance, section 3.2 of \cite{JansenBook}, it's quite remarkable to observe that in case of mixed strategy and hybrid EAs it can be used for intelligent design guidance as well as to understand the success behind this novel kind of EAs. The current paper presents only preliminary and highly simplified analysis that may be altered and improved in a number of ways. For instance, the generalized schema theorem~\ref{genSchemaThm} motivates runtime analysis based on the ideas in \cite{MitavJunSchedule} in place of drift analysis methodology to design and analyze hybrid and/or mixed strategy EAs where the runtime to encounter a satisfactory solution is polynomial with overwhelmingly high probability. This work is postponed for the future research. Nonetheless, the authors believe that the core ideas of designing the collections of generalized schemata (see definition~\ref{schemaDefn}) based on the auxiliary fitness levels in a similar manner to the way it's been done in section~\ref{RecInvarAndDriftSect} and then designing the families of recombination and mutation transformations based on the corresponding families of generalized schemata, opens the door to understanding and designing efficient hybrid and mixed strategy EAs.

\paragraph*{Acknowledgment}
This work has been supported by the EPSRC under Grant No. EP/I009809/1.

\end{document}